\documentclass[11pt]{article}
\usepackage[final]{acl}
\usepackage{times}
\usepackage[T1]{fontenc}
\usepackage{amsmath,amssymb,amsthm}
\usepackage{booktabs}
\usepackage{graphicx}
\graphicspath{{figures/}}
\usepackage{microtype}
\usepackage{listings}
\usepackage{xcolor}
\usepackage{url}

\newtheorem{proposition}{Proposition}

\newcommand{\ratk}[1]{R@#1}
\newcommand{\feone}{\texttt{fusion-embedding-1}}
\newcommand{\fetwo}{\texttt{fusion-embedding-2}}

\title{Modality-Gated Deep Adapters: Adding a Modality to a Frozen\\
Embedding Model with Exact Preservation}

\author{%
  Abdul Basit Tonmoy$^{1,2,3}$\thanks{\,Corresponding author.} \quad
  Kazi Fardinul Hoque$^{2}$ \quad
  Md.\ Shahrier Islam Arham$^{1,2}$ \quad
  Arman Luthra$^{1,2}$ \\[3pt]
  $^{1}$Eximius Labs \quad $^{2}$Wabash College \quad
  $^{3}$Skop Intelligence Co. \\[3pt]
  \texttt{atonmoy27@wabash.edu} \quad \texttt{kfhfar@amazon.com} \\[2pt]
  \texttt{marham27@wabash.edu} \quad \texttt{aluthra26@wabash.edu}}

\begin{document}
\maketitle

\begin{abstract}
Multimodal embedding models are deployed at scale: retrieval indices,
benchmark results, and behavioral audits all depend on the base model's
exact outputs. Extending such a model to a new modality with existing
parameter-efficient methods silently changes those outputs; LoRA-style
adaptation rewrites the text path whether or not the weights are merged,
invalidating every stored embedding. We propose \emph{modality-gated
deep adapters}: bottleneck adapters attached to every decoder layer of a
frozen multimodal embedding LLM, grouped into per-modality packs that
execute only while their own modality is being encoded. The result is a
modality added with zero change to existing outputs: inputs no pack
claims traverse the base model's own computation graph, bit-for-bit
unchanged, and co-loaded packs compose with an exact-zero isolation
matrix. Both properties are stated as propositions, hold after arbitrary
training rather than only at initialization, require no task labels or
routing metadata at inference, and are verified by exact-equality tests
on the released checkpoints. On one frozen 2B base, the audio pack
(injected as connector tokens) improves audio-to-text \ratk{10} by
$+3.4$ to $+5.4$ points over an identically trained control, positive at
every seed and reproduced at eleven times the data; the thermal pack,
reusing the base's own frozen vision path, clears its pre-registered
acceptance gate roughly sevenfold at every seed and lifts
thermal-to-text \ratk{10} from 0.224 to 0.785. A motivating encoder swap
locates the missing capacity: an external audio encoder that outranks
Whisper-family encoders in CLAP-style comparisons loses by 16 \ratk{10}
points inside the frozen LLM, so the capacity belongs in the layers,
exactly where the gated adapters place it. We release the audio model,
the thermal pack, and the training, evaluation, and invariance test
suites: models at \url{https://huggingface.co/EximiusLabs}, code on
GitHub.
\end{abstract}

\begin{figure}[t]
\centering
\includegraphics[width=\columnwidth]{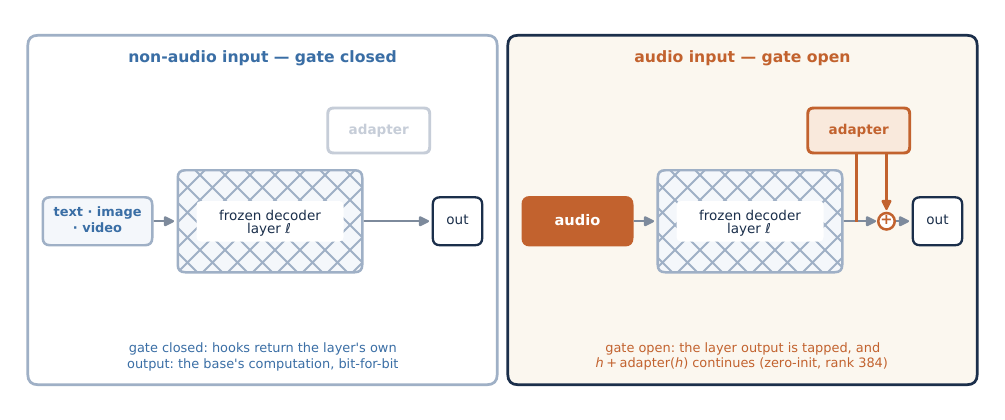}
\caption{One frozen decoder layer with a gated adapter. Left: on an
input whose modality no pack claims (text, image, video), the hook
returns the frozen layer's output before any adapter arithmetic
executes, so the computation graph is the released base model's own and
the output is bit-for-bit the base's
(Proposition~\ref{prop:exact}); ungated methods such as LoRA instead
execute on this path and change it. Right: on an input of the pack's
own modality (here audio), the bottleneck adapter is added to the
residual stream and trained. Each added modality carries its own pack
and gate, keyed on the input rather than on a task label, and
co-loaded packs are mutually invisible, so packs compose with
exact-zero isolation (Table~\ref{tab:compose}).}
\label{fig:mechanism}
\end{figure}

\section{Introduction}
\label{sec:intro}

An embedding model in production is not just weights; it is a contract.
A deployed embedding model's outputs are indexed at scale, benchmark
tables record its published behavior, and downstream systems are tuned
to its geometry.
When such a model gains a new modality, the operative question is not
only how well the new modality performs but what happens to that
contract. Retraining produces a new model and a full re-indexing bill.
Parameter-efficient adaptation appears gentler, but it is not: LoRA
updates, merged or unmerged, execute on every input, so the adapted
model is a different function on the modalities users already depend on
\citep{hu2022lora,omniembedaudio2026}. The change is silent; nearest-neighbor
structure shifts without any error being raised.

Freezing the base is the obvious response, and a family of methods
freezes most of it. Yet frozen-backbone systems in the literature treat
freezing as a compute saving, not a behavioral guarantee.
Omni-Embed-Audio \citep{omniembedaudio2026} keeps its backbone frozen but
trains LoRA in the attention of every layer; the adapters fire on the
text path as well, and the deployed function on text is far from the
base's (the untrained backbone retrieves at roughly chance, the adapted
one near state of the art, so the text path was necessarily rewritten).
Methods that install zero-initialized new capacity, such as gated
cross-attention in Flamingo \citep{alayrac2022flamingo} and block
expansion in LLaMA~Pro \citep{wu2024llamapro}, guarantee identity only at
initialization; the new modules train away from zero on all inputs.
Such methods report drift on the base's own tasks once the added modules
train, so a preservation property that holds only at initialization does
not survive continued training.

This paper presents a mechanism whose preservation guarantee is exact and
permanent. \emph{Modality-gated deep adapters} add a bottleneck adapter
to the residual stream of every decoder layer of a frozen multimodal
embedding LLM, and a binary gate, keyed on the presence of audio in the
input rather than on any task label, decides whether the adapter branch
executes at all (Figure~\ref{fig:mechanism}). While audio is encoded the
adapters are active and
trained; for any input without audio every adapter hook returns the
frozen layer's output before any adapter arithmetic, so the executed
computation graph is the base model's own and the output is bit-for-bit
identical to the released base (Proposition~\ref{prop:exact}). The
guarantee is a property of the
computation, not of the training trajectory: it holds after arbitrary
training, forever, and it is machine-verified at three points:
exact-equality unit tests, a parameter-drift assertion after every
training run, and a release smoke test on the packaged public
checkpoint.

Three further results make the mechanism more than a safety feature.
First, a motivating probe (\S\ref{sec:probe}) shows that the natural
alternative, upgrading the external audio encoder, fails: an encoder
that ranks above Whisper-family encoders in CLAP-style comparisons loses
by 16 \ratk{10} points inside the frozen-LLM architecture. The
bottleneck is in-layer capacity, precisely what the gated adapters add.
Second, a controlled experiment (\S\ref{sec:gate-exp}) isolates the
mechanism: with data, recipe, and connector held fixed, gated adapters
improve audio-to-text \ratk{10} at every seed tested (paired deltas
$+3.4$ to $+5.4$), the gain scales with
adapter rank, and its point estimate reproduced at eleven times the
data. The released model improves the point estimate in every
text-to-audio retrieval cell on AudioCaps, Clotho, and VGGSound over its
non-adapter predecessor. Third, the construction is not an
audio-specific trick (\S\ref{sec:thermal}): a thermal pack added to the
same frozen base through a different injection route, the base's own
vision path, clears its pre-registered gate at every seed, and
co-loading the audio and thermal packs reproduces the base bitwise on
every text, image, and video readout.

Our contributions:
\begin{itemize}
\item \textbf{Mechanism.} Modality-gated deep adapters: post-hoc
  extension of a frozen multimodal embedding LLM to new modalities, with
  per-modality packs of in-layer trainable capacity that execute only on
  their own modality's inputs (Figure~\ref{fig:mechanism},
  \S\ref{sec:method}). We add audio and thermal to one frozen
  base through two different injection routes and show the packs compose
  (\S\ref{sec:thermal}).
\item \textbf{Property and verification.} Exact preservation, stated as
  Propositions~\ref{prop:exact} and~\ref{prop:exclusive}
  (computation-graph identity for inputs no pack claims,
  with bitwise equality as the verified corollary) and enforced by
  exact-equality tests, gradient-isolation tests, a multi-pack
  exclusivity invariant, and a released smoke test
  (Tables~\ref{tab:verify} and~\ref{tab:compose}); the guarantee is
  input-conditional and permanent, which is
  strictly stronger than identity at initialization
  \citep{alayrac2022flamingo,wu2024llamapro} or task-labeled routing
  \citep{rusu2016pnn,zhang2020sidetuning} (\S\ref{sec:property},
  \S\ref{sec:preservation-results}).
\item \textbf{Controlled evidence, two modalities.} A same-recipe audio
  experiment ($+3.4$ \ratk{10}, rank-scaled, reproduced at full scale;
  Table~\ref{tab:gate}), a
  caption-free thermal probe that clears its pre-registered gate at every
  seed (\S\ref{sec:thermal}), and a publishable negative (a
  leaderboard-stronger audio encoder
  degrades the system by 16 \ratk{10} points; Table~\ref{tab:towerswap}),
  which together locate the
  frozen-LLM bottleneck in the layers, not the encoder.
\item \textbf{Artifact.} The released \fetwo{} audio model and the
  \texttt{ember} thermal pack (each 2--3\% trained parameters over a
  byte-frozen 2B base), training and evaluation code, and the invariance
  and composability test suites
  (\S\ref{sec:conclusion}).
\end{itemize}

\section{Is the audio encoder the bottleneck?}
\label{sec:probe}

The architecture we extend routes audio through a frozen encoder tower
and a small trained connector into a frozen decoder-LM embedding model
(details in \S\ref{sec:setup}). Its audio-text retrieval trails
specialist dual encoders, and the most natural diagnosis is the audio
encoder: the tower is Whisper-derived
\citep{radford2022whisper,xu2025qwen25omni}, and controlled comparisons
under shallow projections rank Whisper-family encoders well below
sound-event encoders for sound retrieval. GLAP's encoder study
\citep{dinkel2025glap} puts CED-Base at 58.6 AudioCaps text-to-audio
mAP@10, Dasheng at 55.8, and Whisper-Base at 46.5 under one recipe.

We ran the implied experiment. Two arms, identical in data (45K
AudioCaps-only), steps (800), recipe, and connector, differ only in the
frozen tower supplying audio frames: the base's own Qwen2.5-Omni tower
(Whisper-family, co-trained to feed a Qwen LM) or Dasheng-base
\citep{dinkel2024dasheng}, the versatile winner of GLAP's comparison.
The external ranking inverts inside the splice: the Omni tower reaches
audio-to-text \ratk{10} 0.631 against 0.469 for Dasheng, with the same
ordering in the reverse direction (0.684 vs.\ 0.538 text-to-audio;
Table~\ref{tab:towerswap} and Figure~\ref{fig:probe}, protocol in
Appendix~\ref{app:probe}). The
towers are not parameter-matched (86M vs.\ ${\sim}$640M), so the probe
does not isolate co-training from scale; what it establishes is that
encoder rankings measured under shallow projections do not transfer to
architectures that splice tokens into a frozen LLM. Two adjacent
observations reinforce the conclusion: a linear probe of the Omni tower
on ESC-50 \citep{piczak2015esc50} rises monotonically to its last layer
(0.924, with a learned average over all layers worse at 0.893; from
training-run logs), so there is no hidden mid-stack feature to tap
\citep{gong2023whisperat}; and the one controlled same-connector study we
know of finds that adapting a Whisper-family encoder buys nothing for
audio captioning (LoRA 44.8 vs.\ frozen 44.8 SPIDEr-FL, full fine-tuning
worse at 44.0) while a frozen event encoder reaches 49.6
\citep{interspeech2024cedwhisper}.

\begin{table}[t]
\centering
\small
\setlength{\tabcolsep}{3.4pt}
\begin{tabular}{lcc}
\toprule
Tower & A$\to$T \ratk{10} & T$\to$A \ratk{10} \\
\midrule
Qwen2.5-Omni (${\sim}$640M) & \textbf{0.631} & \textbf{0.684} \\
Dasheng-base (86M) & 0.469 & 0.538 \\
\bottomrule
\end{tabular}
\caption{An encoder that outranks Whisper-family towers in
shallow-projection comparisons \citep{dinkel2025glap} loses by 16
points inside the frozen-LLM splice. Matched 45K/800-step arms,
identical recipe, each arm on frames from its own frozen tower. The
towers are not parameter-matched; the probe establishes non-transfer of
the external ranking, not encoder superiority per parameter.}
\label{tab:towerswap}
\end{table}

\begin{figure}[t]
\centering
\includegraphics[width=\columnwidth]{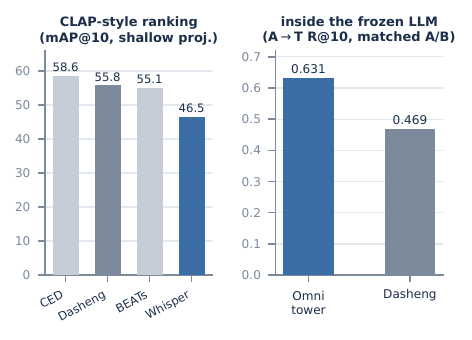}
\caption{The ranking inversion behind Table~\ref{tab:towerswap}. Left:
under shallow projections, Dasheng outranks Whisper-family encoders on
sound retrieval (AudioCaps text-to-audio mAP@10, as reported by
\citealp{dinkel2025glap}). Right: spliced into the frozen LLM under a
matched recipe, the base's co-trained Whisper-family tower wins by 16
\ratk{10} points. External encoder rankings do not transfer to this
architecture class.}
\label{fig:probe}
\end{figure}

If the encoder is not the bottleneck, the natural remaining suspect is the frozen
LM itself: every layer of it was trained on text and images, and all
audio understanding must squeeze through a 16.4M input-side connector
into layers that never learned to process audio. Capacity must be added
inside the layers, on the audio path. Adding it the standard way (LoRA
in the attention) rewrites the model on every input; adding it behind a
modality gate does not. That is the mechanism.

\section{Modality-gated deep adapters}
\label{sec:method}

This section defines the mechanism and proves its central property,
exact preservation (Proposition~\ref{prop:exact}): an input no pack
claims executes the base model's own computation graph, so its output
is bit-for-bit the released base's.

\subsection{Setup and notation}
\label{sec:setup}

The base model is Qwen3-VL-Embedding-2B
\citep{qwen3vlembedding}, a 28-layer decoder-LM embedding model with
hidden width $d=2048$, last-token pooling, and Matryoshka output
truncation \citep{kusupati2022mrl}; it embeds text, images, and video in
one space and is kept byte-frozen throughout (a regression guard asserts
that no base parameter changed after every run). Audio enters through
the frozen Qwen2.5-Omni audio tower \citep{xu2025qwen25omni} and a
trained 16.4M perceiver-resampler connector that writes 64 audio tokens
into the base's input stream at placeholder positions; this
connector-only configuration is the released \feone{}~\citep{tonmoy2026fusion} and serves as the
architecture our adapters extend. Training is symmetric InfoNCE
\citep{oord2018infonce} over audio-caption pairs against cached frozen
text targets. Let $h_\ell \in \mathbb{R}^{n \times d}$ be the hidden
states after decoder layer $\ell$, and let $g(x) \in \{0,1\}$ indicate
whether input $x$ contains audio tokens: $g$ is a fixed function of the
input's composition, not a learned or task-supplied router.

\subsection{The mechanism}
\label{sec:mechanism}

At every decoder layer $\ell$ we attach a bottleneck adapter
\citep{houlsby2019adapters}
\begin{equation}
A_\ell(h) = W^{\mathrm{up}}_\ell\,
  \sigma\!\bigl(W^{\mathrm{down}}_\ell\, \mathrm{LN}(h)\bigr),
\label{eq:adapter}
\end{equation}
with $W^{\mathrm{down}}_\ell \in \mathbb{R}^{r \times d}$,
$W^{\mathrm{up}}_\ell \in \mathbb{R}^{d \times r}$, $\sigma$ the SiLU
nonlinearity, and $\mathrm{LN}$ a LayerNorm. The adapter output is added
to the residual stream, but only through a gate:
\begin{equation}
h_\ell \leftarrow
\begin{cases}
h_\ell + A_\ell(h_\ell) & \text{if } g(x) = 1,\\[2pt]
h_\ell & \text{if } g(x) = 0.
\end{cases}
\label{eq:gate}
\end{equation}
The gate is implemented as a depth-counted context manager held open by
the model wrapper around audio encodes; the adapters are attached with
forward hooks and owned by the wrapper, never registered under the base
module, so the base's parameter snapshot and \texttt{state\_dict} remain
adapter-free. Figure~\ref{fig:mechanism} shows the mechanism; the
implementation of the gated hook is short enough to give in full
(Figure~\ref{fig:code}). At rank $r=384$ over 28 layers the adapters
total 44.2M parameters, which together with the 16.4M connector gives
60.6M trained parameters, about 3\% of the 2B base.

\begin{figure}[t]
\begin{lstlisting}[language=Python]
def _make_hook(adapter, gate):
  def hook(_module, _inputs, output):
    if not gate.active:
      # keep original output:
      # bitwise no-op
      return None
    if isinstance(output, tuple):
      h = output[0]
      return ((h + adapter(h),)
              + tuple(output[1:]))
    return output + adapter(output)
  return hook
\end{lstlisting}
\caption{The gated forward hook from the released code (type
annotations and comments shortened; indentation compressed for the
column). When the gate is closed the hook returns before any adapter
computation, so the frozen layer's output is used unchanged; the tuple
branch handles decoder layers that return \texttt{(hidden, ...)}.}
\label{fig:code}
\end{figure}

\subsection{Exact preservation}
\label{sec:property}

The mechanism generalizes past a single modality. A \emph{pack} $m$ is a
full per-layer adapter stack (Eq.~\ref{eq:adapter}) with its own binary
gate $g_m(x)$; several packs may share the same frozen layers, each
registering its own gated hook. The released model carries an audio pack,
and \S\ref{sec:thermal} adds a thermal pack to the same base. Let
$f_{\mathrm{base}}$ be the released base and $f$ the model with any set of
packs attached.

\begin{proposition}[Exact preservation]
\label{prop:exact}
For every input $x$ with $g_m(x)=0$ for all packs $m$, $f$ executes
the same operations as $f_{\mathrm{base}}$, on the same weights, in the
same order; hence $f(x) = f_{\mathrm{base}}(x)$ exactly.
\end{proposition}

\begin{proof}
The only modifications to the base are the packs' layer hooks. With every
gate closed, each hook returns the frozen layer's
output before any adapter operation enters the computation
(Eq.~\ref{eq:gate}, Figure~\ref{fig:code}); no adapter parameter,
activation, or cast of any pack participates. The executed computation
graph is therefore $f_{\mathrm{base}}$'s own, and the outputs coincide
exactly.
\end{proof}

A companion invariant makes the multi-pack case precise.

\begin{proposition}[Single-gate exclusivity]
\label{prop:exclusive}
Fix an input $x$ and suppose exactly one pack $m^\star$ has
$g_{m^\star}(x)=1$ while every other gate is closed. Then $f(x)$ equals
the output of the model carrying pack $m^\star$ alone: co-loaded packs are
mutually invisible, because each closed pack's hook is a bitwise no-op.
\end{proposition}

We scope the invariant deliberately to \emph{at most one gate active per
forward}. A pathological input that opened two gates at once would have
both packs add to the residual stream, an interaction we neither train nor
claim; because each modality is encoded through its own entry point (below),
this case does not arise in normal use and we exclude it
(\S\ref{sec:limitations}). What the invariant covers is the deployment mix:
an audio encode fires only the audio pack, a thermal encode only the thermal
pack, and any text, image, or video forward fires neither, reproducing the
base bitwise (\S\ref{sec:thermal}, Table~\ref{tab:compose}).

Two scope statements make the claim precise. First, bitwise output
equality follows under any fixed inference configuration (same weights,
precision, kernels, and batch composition); across configurations the
model is exactly as deterministic as the base, because on inputs no pack
claims it \emph{is} the base's computation. Second, the property is
input-conditional: a sequence that a pack does claim is adapted
as a whole, including its text tokens. Such inputs are new behavior by
definition; the base model has no audio (or thermal) interface, so there
is no base behavior to regress. We deliberately state the property as
``the branch never executes'' rather than ``the branch adds zero,'' since
$h + 0 \cdot A(h)$ would still change the graph and invite
floating-point concerns; a hook that returns early does not.

Each gate is \emph{declared by the encode entry point} rather than inferred
from tensor contents. This matters most for thermal: a thermal image is a
single channel replicated to three, byte-indistinguishable from an ordinary
RGB image at the vision interface, so no content test could route it. The
thermal encode is instead the only code path that opens the thermal gate,
and it closes it on exit; every image, video, and text embed runs outside
that scope with the gate closed. For audio the declaration is reinforced by
content: audio tokens are present only on the audio path, and the public
text-encode and image-embed entry points hard-fail if invoked while the
audio gate is open (in the training and released inference code
respectively). The released inference guard covers the audio gate; thermal
preservation rests on the encode-scoping discipline just described, which
the isolation matrix of \S\ref{sec:thermal} verifies bitwise. The gates
$g_m$ in Propositions~\ref{prop:exact} and~\ref{prop:exclusive} are thus
properties the encode entry points guarantee, not heuristics over inputs.

Proposition~\ref{prop:exact} is stronger than the preservation notions
in prior expansion methods on two independent axes. On the \emph{time}
axis, zero-initialized capacity \citep{alayrac2022flamingo,
bachlechner2021rezero,wu2024llamapro} matches the base only at
initialization; our guarantee is invariant to training, because training
only ever touches the gated branch. On the \emph{routing} axis,
progressive networks \citep{rusu2016pnn} and side-tuning
\citep{zhang2020sidetuning} preserve frozen components but select the
active path with task labels supplied at inference; our gate is a
function of the input alone. Unmerged LoRA preserves base
\emph{parameters}, but its adapters execute on every path, so the
deployed function changes on all inputs \citep{hu2022lora,
omniembedaudio2026}; parameter preservation without computation
preservation is the weaker guarantee, and the distinction is exactly
what re-indexing costs measure. Section~\ref{sec:preservation-results}
verifies the property empirically, including the negative control (with
the gate open, outputs change; that visible change is what ungated
adaptation applies to every input silently).

\subsection{Identity at initialization}
\label{sec:identity}

To ensure that at initialization the adapted model produces exactly the
outputs of the base model on \emph{all} inputs, including audio, the
up-projections $W^{\mathrm{up}}_\ell$ are zero-initialized, following the
near-identity principle of \citet{houlsby2019adapters} and the
zero-init gating of \citet{alayrac2022flamingo} and
\citet{bachlechner2021rezero}: at step zero
$A_\ell(h) = 0$ and the audio path behaves as the connector-only
architecture, so training starts from the predecessor model rather than
from a perturbed one, which stabilizes optimization and makes
warm-starting from a connector-only checkpoint well-defined. The two
properties are complementary and should not be conflated: identity at
initialization holds for all inputs and is temporary; exact preservation
holds for inputs that no pack claims (audio-free and thermal-free) and is
permanent.

\subsection{The gate must span forward and backward}
\label{sec:gc}

One implementation subtlety carries the guarantee in practice. With
gradient checkpointing, layer forwards re-run during the backward pass.
A gate that is open for the forward pass but closed by backward time
silently drops the adapters from the recomputed graph: gradients are
wrong and no error is raised, unless the checkpointing implementation
detects the mismatch. We therefore hold the gate open across forward and
backward of every audio step (the depth-counted design makes the
trainer's scope and the encoder's inner scope compose), rely on
non-reentrant checkpointing to raise loudly on any mismatch, and lock the
failure mode in with a dedicated test. The subtlety is also evidence
that the property has content: naive gating does not provide the
guarantee, and a correct implementation must be verified, which is why
the test suite of \S\ref{sec:preservation-results} exists. A second
guard closes the loop from the other side: a text encode issued while
the gate is open raises immediately, so the frozen text targets used by
the contrastive objective can never be silently adapted.

\subsection{Design choices}
\label{sec:design}

Adapters attach at every layer because the probe of \S\ref{sec:probe}
locates the deficit in the frozen stack generally, and because layer
subsets were not needed to pass the pre-registered gate; layer-placement
ablations remain open (\S\ref{sec:limitations}). The rank is chosen
empirically ($r=384$; \S\ref{sec:gate-exp} sweeps it). The adapter
computes in fp32 and casts back to the stream's dtype for stability at
bf16 scale. The gate keys on audio because audio is the modality being
added; nothing in the construction is audio-specific, and the same
gating applies to any modality appended to a frozen base.

\section{Experiments}
\label{sec:experiments}

\subsection{Setup}
\label{sec:exp-setup}

\paragraph{Training.} All runs pair the frozen base and tower with the
16.4M connector; adapter runs add gated adapters per
\S\ref{sec:method}. The controlled experiment trains on a 45K
AudioCaps-only corpus \citep{kim2019audiocaps} for 800 steps; full-scale
runs train on 518{,}183 audio-caption pairs (a 592K web corpus after
removing 73{,}716 clips whose metadata carried no sound content) for
3{,}900 steps at effective batch 1{,}024, with a full-corpus frozen-text
negative bank and large-scale loss terms. What that corpus does and does
not teach the audio path, and why its structure rather than its size
governs which attributes survive, is analyzed
separately~\citep{tonmoy2026axis}. A 400-step in-domain
AudioCaps fine-tune follows, as in the predecessor. Precision lineage
is matched end to end (train, fine-tune, and score at one precision);
mixing lineages costs about 2 points and is a documented trap. Full
hyperparameters are in Appendix~\ref{app:hyper}.

\paragraph{Evaluation.} Two protocols are used, never mixed. The
\emph{in-run protocol} is the automated AudioCaps rescore every training
job runs at its own precision; it is used for matched-arm comparisons.
The \emph{release protocol} (bf16, the base's native input template,
five-reference min-rank scoring on AudioCaps-883; Clotho v2.1
\citep{drossos2020clotho} strictly zero-shot; VGGSound-696
\citep{chen2020vggsound} for cross-modal cells) is used for released
checkpoints. Clotho never appears in training; VGGSound is blacklisted
from ingestion, and 5{,}929 Clotho-overlapping FreeSound clip ids are
excluded from the web corpus. At these pool sizes a recall proportion
carries a binomial standard error of roughly 1.4--1.5 points, so we read
single deltas below ${\sim}2$ points as parity. The controlled gate
experiment is replicated across three seeds (\S\ref{sec:gate-exp}); the
full-scale pretrain pair and the release checkpoints are single runs.

\subsection{The controlled gate experiment}
\label{sec:gate-exp}

The causal question is whether gated in-layer capacity improves audio
retrieval when everything else is held fixed. Three arms share data,
steps, recipe, connector, and tower, and differ in one axis
(Table~\ref{tab:gate}): no adapters, gated adapters at rank 128, gated
adapters at rank 384. The acceptance gate was pre-registered at
$\geq{+}3$ audio-to-text \ratk{10}.

\begin{table}[t]
\centering
\footnotesize
\setlength{\tabcolsep}{2.2pt}
\begin{tabular}{@{}lccccc@{}}
\toprule
Arm & Trained & \ratk{1} & \ratk{10} & T$\to$A & Loss \\
\midrule
Control (no adapt.) & 16.4M & 0.202 & 0.631 & 0.684 & 4.70 \\
\ + gated, $r{=}128$ & 31.2M & 0.217 & 0.656 & --- & --- \\
\ + gated, $r{=}384$ & 60.6M & \textbf{0.229} & \textbf{0.665} &
\textbf{0.708} & \textbf{4.46} \\
\bottomrule
\end{tabular}
\caption{Gated in-layer capacity passes the pre-registered acceptance
gate ($\geq{+}3$ \ratk{10}) at $+3.4$, improves every measured
direction, and has not saturated at the largest rank trained. The
controlled experiment: 45K AudioCaps-only, 800 steps, identical recipe,
one change per row; in-run protocol; audio-to-text unless marked; final
training loss in the last column. Rank 384 is the released
configuration. Dashes: not recorded.}
\label{tab:gate}
\end{table}

Gated adapters at rank 384 improve audio-to-text \ratk{10} by $+3.4$
points and every other measured cell, with a lower final training loss;
rank 128 recovers $+2.5$, so the gain scales with adapter capacity and
had not saturated. Both deltas exceed the ${\sim}2$-point single-seed
parity threshold of \S\ref{sec:exp-setup}. The comparison isolates the
adapters as the only changed component; it does not equalize
trained-parameter count across arms. Two observations address the
budget question directly. First, spending a comparable increment at the
input instead does not help: in the predecessor's width study, growing
the connector from $d_r{=}384$ to $d_r{=}512$ matched or
\emph{reduced} held-out retrieval at the two scales tested (0.479
vs.\ 0.481 at 131K; 0.675 vs.\ 0.717 \ratk{10} at 484K) despite better
training metrics. Second, that width increase and the rank-128 arm are
similarly sized parameter increments (11.8M and 14.8M); spent in-layer
behind the gate the increment gains $+2.5$, spent on connector width it
gains nothing. Where the parameters live matters more than how many
there are: input-side width was already past its knee, while in-layer
capacity was absent entirely.

\paragraph{Seed replication.} Re-running both arms of
Table~\ref{tab:gate} with a seeded variant of the identical recipe (the
seed enters the training run's resume key; the original arms predate the
seed parameter) replicates the effect. Across three seeds the paired
audio-to-text \ratk{10} delta is $+3.4$/$+4.6$/$+5.4$ (mean $+4.5$;
control 0.631/0.608/0.595, adapters 0.665/0.655/0.649); text-to-audio
\ratk{10} deltas are $+2.4$/$+4.6$/$+5.2$ and audio-to-text \ratk{1}
deltas $+2.7$/$+3.9$/$+5.1$. Every retrieval direction is positive at
every seed, and the pairing published in Table~\ref{tab:gate} is the
most conservative of the three.

At eleven times the data the point estimate reproduces. With the
518{,}183-pair filtered corpus and large-scale loss terms, the adapter
pretrain reaches in-run \ratk{10} 0.708 against 0.674 for the closest
available baseline, the 592K raw-corpus run without adapters or the
loss terms (matched evaluation protocol): $+3.4$ again. The full-scale
pair differs on three axes (adapters, loss terms, and corpus
filtering); the controlled isolation is Table~\ref{tab:gate}, and the
full-scale pair shows the gain's magnitude survives scale and recipe
maturation.

\subsection{Release results: the family delta}
\label{sec:release-results}

The released adapter model improves the point estimate in every
text-to-audio cell over its connector-only predecessor under the
release protocol (Table~\ref{tab:family}). The pattern is
directional and consistent with the mechanism: in-layer capacity lets
the audio representation organize for discrimination against text
queries (\ratk{10}: AudioCaps $+2.9$, Clotho $+2.2$,
VGGSound $+3.6$; VGGSound \ratk{1} $+5.3$; the AudioCaps \ratk{1} gain
of $+1.2$ and the Clotho \ratk{1} gain of $+1.5$ are below the parity
threshold), at a small cost in
audio-to-text top-1 on the caption benchmarks. Audio-to-text \ratk{10}
on AudioCaps is statistical parity (0.743 vs.\ 0.741). Against external systems the
released model remains behind specialist dual encoders on in-domain
audio-text retrieval (GLAP reports 0.544/0.911 AudioCaps audio-to-text
\ratk{1}/\ratk{10} \citep{dinkel2025glap}); those systems train both
towers on audio-text data and serve exactly one modality pair, and no
result in this paper claims otherwise. What the adapter model uniquely
retains is the rest of the space: its text, image, and video behavior is
the base model's, bitwise, which no specialist or LoRA-adapted system
provides.

\begin{table}[t]
\centering
\footnotesize
\setlength{\tabcolsep}{2.6pt}
\begin{tabular}{@{}lcc@{}}
\toprule
Cell & Predecessor & Ours \\
\midrule
\multicolumn{3}{l}{\emph{AudioCaps}} \\
\ A$\to$T \ratk{1} / \ratk{10} & \textbf{0.332} / 0.741$^{\dagger}$ &
  0.302 / \textbf{0.743}$^{\dagger}$ \\
\ T$\to$A \ratk{1} / \ratk{10} & 0.280$^{\dagger}$ / 0.746 &
  \textbf{0.292}$^{\dagger}$ / \textbf{0.775} \\
\midrule
\multicolumn{3}{l}{\emph{Clotho (zero-shot)}} \\
\ A$\to$T \ratk{1} / \ratk{10} & \textbf{0.135}$^{\dagger}$ /
  \textbf{0.433}$^{\dagger}$ &
  0.127$^{\dagger}$ / 0.421$^{\dagger}$ \\
\ T$\to$A \ratk{1} / \ratk{10} & 0.136$^{\dagger}$ / 0.460 &
  \textbf{0.151}$^{\dagger}$ / \textbf{0.482} \\
\midrule
\multicolumn{3}{l}{\emph{VGGSound-696}} \\
\ A$\to$T \ratk{1} / \ratk{10} & \textbf{0.213}$^{\dagger}$ / 0.625 &
  0.211$^{\dagger}$ / \textbf{0.665} \\
\ T$\to$A \ratk{1} / \ratk{10} & 0.213 / 0.645 &
  \textbf{0.266} / \textbf{0.681} \\
\bottomrule
\end{tabular}
\caption{The adapters improve the point estimate in every text-to-audio
cell and the VGGSound audio-to-text \ratk{10}, and cede audio-to-text
\ratk{1} on the caption benchmarks. Release-protocol comparison against
the connector-only predecessor (\feone{} v0.3), matched bf16 lineage
and native input template; Clotho strictly zero-shot. Bold: better
point estimate per cell; $^{\dagger}$marks both members of a pair whose
difference is below the ${\sim}2$-point single-seed parity threshold
(\S\ref{sec:exp-setup}), read as parity.}
\label{tab:family}
\end{table}

The emergent cross-modal direction is also present. With zero audio-image
training pairs, audio-to-image \ratk{10} on VGGSound reaches 0.443 for the
adapter pretrain (bare input template) and 0.392 for its native-template
fine-tuned successor, both far above chance. We do not read a gain from
the predecessor's native-template release into this, since template
differences alone shift retrieval by about two points; the point is only
that the adapters retain the emergent audio-image bridge that the
audio-to-text alignment rides on.

\subsection{Preservation verification}
\label{sec:preservation-results}

Every check that must return zero returns exactly zero, and every
control that must differ does (Table~\ref{tab:verify}); all checks are
released with the code.

\begin{table}[t]
\centering
\footnotesize
\setlength{\tabcolsep}{3pt}
\begin{tabular}{@{}p{0.78\columnwidth}r@{}}
\toprule
Check & Result \\
\midrule
Base parameter drift, max $|\Delta|$, every training run & 0 \\
Text forward $\Delta$ vs.\ base, gate closed (unit test) & 0 \\
Text/image $\Delta$, hooks attached vs.\ removed, released
checkpoint (smoke test) & 0 \\
Zero-init adapter stack under an \emph{open} gate & 0 \\
Audio forward $\Delta$, trained adapters, gate open & $\neq 0$ \\
Gradients reaching base or tower parameters & none \\
Gate not spanning checkpointed backward & raises \\
\bottomrule
\end{tabular}
\caption{Verification of Proposition~\ref{prop:exact} and its
supporting machinery. Zeros are exact (maximum absolute difference under
matched execution), not tolerances. The open-gate row is the negative
control: outputs must change when the branch executes, and that visible,
gated change is what ungated adaptation applies to every input.
Artifact-level detail in Appendix~\ref{app:verification}.}
\label{tab:verify}
\end{table}

The exact-equality unit tests encode identical text through the
base and through the adapter-attached model and assert bitwise-equal
activations (maximum absolute difference zero, not small). The release
smoke test embeds identical text and image inputs with the adapter
hooks attached (gate closed) and with the hooks removed, through the
public loading path, and requires bitwise-equal outputs on both
modalities. The regression guard's drift assertion returned zero on
every training run reported here. The controls that must differ, do: an
open gate changes audio-gated forwards, and only adapter and connector
parameters receive gradients. A downstream consequence makes the
guarantee visible in retrieval tables: text-image retrieval cells for
the predecessor and the adapter model are identical because the
underlying vectors are identical to the last bit, and the base's
published text/image/video benchmark results
\citep{qwen3vlembedding,meng2025vlm2vecv2} remain true of the extended
model by construction rather than by re-measurement.

\section{A second modality: thermal}
\label{sec:thermal}

The construction claims to be modality-agnostic. We test that claim by
adding a second, deliberately dissimilar modality to the same frozen base
through a \emph{different injection route}, and by verifying that the two
packs coexist without disturbing each other or the base.

\paragraph{A second injection route.} Audio enters as connector tokens: an
external tower and a trained resampler write new tokens into the input
stream. Thermal imagery enters through the base's \emph{own} frozen vision
path. A single-channel thermal frame is replicated to three channels and
processed by the unchanged Qwen3-VL vision encoder and projector; the
capacity that the frozen decoder lacks for thermal content is supplied by a
thermal pack, a per-layer \texttt{GatedAdapter} stack (Eq.~\ref{eq:adapter})
registered through the multi-pack registry with its own gate. No new tower,
no new tokens. The two routes exercise the mechanism differently: the audio
gate can key on the presence of audio tokens, but a replicated thermal image
is byte-identical to an RGB image at the vision interface, so the thermal
gate \emph{must} be declared by the encode entry point (\S\ref{sec:property}).
Thermal is therefore not ``the same trick twice''; it stresses the part of
the design that does not depend on the new modality being content-detectable.

\paragraph{Controlled gate probe.} Mirroring the audio experiment
(\S\ref{sec:gate-exp}), a caption-free probe isolates the thermal pack: two
arms identical in data, steps, and recipe (LLVIP \citep{jia2021llvip}
thermal-to-visible-twin retrieval, rank 384, 800 steps) differ only in
whether the gated thermal pack is attached. The acceptance gate was
pre-registered at $\geq{+}3$ \ratk{10}. The pack lifts thermal-to-image
\ratk{10} from a frozen $0.165$ to $0.385$/$0.388$/$0.390$ across three
seeds, a paired $\Delta$ of $+22.0$ to $+22.5$ points, clearing the
pre-registered gate by roughly sevenfold at every seed; the zero-shot LLVIP
person classifier is unchanged (95.4, $\Delta=0$ every seed), confirming the
pack adds capacity without disturbing the frozen vision path it rides on.

\paragraph{Released thermal pack (Ember).} The deployed pack is trained on
IR-TD, a corpus of 61{,}320 real thermal images with descriptive captions
(FLIR-resolution frames stripped for release hygiene), thermal-to-text
InfoNCE against cached frozen text targets, rank 384, 3{,}900 steps, three
seeds (Table~\ref{tab:thermal}). Thermal-to-text \ratk{10} on a held-out
split rises from $0.224$ (frozen base) to $0.777$--$0.785$; the emergent
thermal-to-visible direction, trained on no thermal-RGB pairs, more than
doubles over the frozen base ($0.165 \to 0.333$--$0.349$). With the thermal pack active, the base's zero-shot
LLVIP classifier is not degraded (95.4 vs.\ 94.1/94.3/91.8 across seeds, a
one to four point spread); this is a robustness observation on the
gate-open thermal readout, not the preservation guarantee. Exact
preservation is the separate bitwise check: with the pack gates closed,
text and RGB forwards are bitwise-identical to the base
(Table~\ref{tab:compose}). A caption
ablation matches the recaptioning lesson of the family report: at 5{,}000
steps, full descriptive captions reach in-domain \ratk{10} $0.843$ and
preserve cross-domain transfer (LLVIP twin $0.343$), whereas short
class-name captions reach only $0.595$ in-domain and collapse transfer to
$0.089$, below the frozen baseline. Caption richness matters on both axes.

\begin{table}[t]
\centering
\footnotesize
\setlength{\tabcolsep}{3pt}
\begin{tabular}{@{}lcc@{}}
\toprule
Thermal cell & Frozen & Ember (3 seeds) \\
\midrule
T$\to$text \ratk{10} & 0.224 & 0.783/0.785/0.777 \\
T$\to$text \ratk{1} & 0.071 & 0.399/0.412/0.404 \\
T$\to$visible \ratk{10}$^{\ast}$ & 0.165 & 0.333/0.349/0.341 \\
LLVIP ZS acc.\ (\%) & 95.4 & 94.1/94.3/91.8 \\
\bottomrule
\end{tabular}
\caption{The released thermal pack (Ember) lifts thermal-to-text
\ratk{10} from 0.224 to 0.777--0.785 and more than doubles the emergent
thermal-to-visible direction, at every seed. Three seeds versus the
frozen base; thermal-to-text is the IR-TD held-out split; the frozen
column is the same base with the pack's zero-initialized up-projection
(exact identity). $^{\ast}$Thermal-to-visible (LLVIP twin) is emergent:
no thermal-RGB pairs appear in training. The zero-shot classifier row
shows the pack does not degrade the base's thermal zero-shot
classification (a gate-open readout); exact preservation of non-thermal
inputs is the bitwise check in Table~\ref{tab:compose}.}
\label{tab:thermal}
\end{table}

\paragraph{Composability: an exact isolation matrix.} The two packs are
co-loaded on the same frozen decoder layers exactly as they would be served
together: the released audio pack through the shipped single-gate attachment,
the released thermal pack through the multi-pack registry. We then embed one
input per modality and compare, bitwise, against the reference that
Proposition~\ref{prop:exact} or~\ref{prop:exclusive} predicts, reporting the
maximum absolute difference (Table~\ref{tab:compose}). Every preserved cell
is exactly zero, including the video readout, which no prior thermal check
had covered; the audio pack driven through the registry code path also
matches the shipped path to the last bit. As a negative control the thermal
readout with its gate open differs from the frozen base by $0.10$, so the
zeros are not a dead pack. The matrix is the machine-checked form of the
deployment claim: adding thermal to the audio model changes nothing a user
already depends on.

\begin{table}[t]
\centering
\footnotesize
\setlength{\tabcolsep}{3pt}
\begin{tabular}{@{}llc@{}}
\toprule
Readout & Gate(s) open & max $|\Delta|$ \\
\midrule
Text & none & 0 \\
Image (RGB) & none & 0 \\
Video & none & 0 \\
Audio & audio & 0 \\
Thermal & thermal & 0 \\
\midrule
Audio, registry path & audio & 0 \\
Thermal (control) & thermal, vs.\ base & $0.10$ \\
\bottomrule
\end{tabular}
\caption{Every preserved readout is exactly zero: the released audio and
thermal packs, co-loaded on one frozen base, reproduce their reference
bit-for-bit. Each readout is compared
to the reference the propositions predict (text/image/video to the raw base,
audio to the audio-pack-only model, thermal to the thermal-pack-only model);
max $|\Delta|$ is the exact maximum absolute difference, $0$ meaning
bitwise-identical. The last two rows are checks: the audio pack through the
registry vs.\ the shipped single-gate path (must be $0$), and the thermal
readout with its gate open vs.\ the frozen base (must be nonzero, confirming
the pack is active). Mixed inputs that would open two gates in one forward
are outside the guarantee and untested (\S\ref{sec:limitations}).}
\label{tab:compose}
\end{table}

\section{Analysis}
\label{sec:analysis}

\subsection{Why does a stronger encoder hurt?}
\label{sec:why-encoder}

The probe of \S\ref{sec:probe} is not a claim that Dasheng is a weak
encoder; on its own benchmarks it is strong. The frozen LLM, however, is
not a generic readout: it consumes token streams shaped like the
distribution its co-training established, and the Omni tower was trained
jointly with a Qwen LM to produce exactly such streams. A stronger but
foreign encoder presents features the frozen layers never learned to
parse, and no capacity exists downstream to compensate (that capacity is
what \S\ref{sec:method} adds). The reading is consistent with GLAP's own
finding that encoder rankings are recipe-relative
\citep{dinkel2025glap}, and with the captioning result that adapting the
encoder does not close the gap \citep{interspeech2024cedwhisper}.
Compatibility with the consumer dominates quality of the producer.

\subsection{What does the gate cost?}
\label{sec:cost}

On non-audio inputs, nothing: no added operations, no added latency, no
added memory beyond holding the (inactive) adapter weights, and exact
output stability, which is the point. On audio inputs the adapters add
one bottleneck per layer (60.6M trained parameters in total, about 3\%
of the 2B base); audio encoding was already the expensive path
through a 640M tower, and the adapter increment is minor by comparison.
At training time the practical cost is the discipline of
\S\ref{sec:gc}: the gate must be scoped across forward and backward, and
the text-encode guard forces cached text targets. We consider both
constraints features; each converts a silent failure mode into a loud
one.

\subsection{Relation to continual learning}
\label{sec:continual}

On its original modalities the extended model does not approximately
retain the base's behavior; it reproduces it exactly, because those
inputs execute the base's computation. In continual-learning terms, the
mechanism has exact recall of base-model outputs on audio-free inputs
by construction, echoing the frozen-parameter memory arguments of
\citet{houlsby2019adapters} and \citet{rusu2016pnn}, while
avoiding both the task-labeled routing those methods require and the
after-training drift of zero-init expansion \citep{wu2024llamapro}. We
scope the claim deliberately: it concerns embedding outputs on
audio-free inputs, not any broader notion of capability retention, and
mixed audio-text inputs are new behavior with no base counterpart.
Preservation language should be scoped precisely rather than asserted
broadly, and ours is checkable by running the released tests.

\section{Related work}
\label{sec:related}

\paragraph{Parameter-efficient adaptation.} Bottleneck adapters
\citep{houlsby2019adapters} and LoRA \citep{hu2022lora} add small
trained deltas to frozen models; both execute on every input, so the
adapted model is a different function everywhere. Our adapter is
Houlsby's module behind a modality gate; the gate, not the bottleneck,
is what produces the guarantee.

\paragraph{Gated injection into frozen LMs.} Flamingo interleaves
tanh-gated cross-attention into a frozen LM with $\alpha=0$ at
initialization \citep{alayrac2022flamingo}, following zero-init
residual principles \citep{bachlechner2021rezero}; the gates train away
from zero for all inputs. Frozen-tower bridges such as BLIP-2 train a
querying transformer between frozen unimodal models
\citep{li2023blip2}; the deployed embedding function is the bridge's,
not the base's, for every input. Our gate is binary and input-keyed
rather than learned, which is why it can be exact.

\paragraph{Frozen-backbone multimodal embedders.} Closest in setting is
jina-embeddings-v5-omni \citep{jina2026omni}, which keeps a text backbone
frozen and trains only projector-style connectors (about 0.35\% of
weights) to bind image, audio, and video into the text space. This is
preservation \emph{by omission}: nothing is added inside the frozen
layers, so the original space is untouched, but the new modalities are
limited to what shallow projections can align. Our packs preserve the
base equally (the closed-gate path \emph{is} the frozen computation) while
adding trainable capacity \emph{inside} the layers on the new modality's
path, and the audio probe of \S\ref{sec:probe} is direct evidence that
in-layer capacity, not a better projector, is what the frozen-LLM setting
is missing. X-InstructBLIP \citep{panagopoulou2024xinstructblip} aligns
several modalities (image, 3D, audio, video) to a frozen LLM through
per-modality Q-Formers on the input side; preservation there is trivial,
in that the backbone is never touched and no exactness property for an
unmodified path is stated or tested, and, as with all input-side
composition, the new modalities' quality is bounded by what tokens a
frozen stack can consume. Our contribution is complementary: an in-layer,
gated capacity whose bypass carries a machine-checked exactness guarantee
and whose packs compose (\S\ref{sec:thermal}). Both routes coexist in
the released family this work extends: inertial motion (Tremor) and
tactile pressure (Tactus) bind to the same frozen space through external
encoder-plus-projector heads, since a small external encoder suffices for
those signals, while audio and thermal take the in-layer route because
their inputs must traverse the frozen stack itself.\footnote{All packs
are released at \url{https://huggingface.co/EximiusLabs}.} Which route a
modality needs is decided by its path through the base; the packs here
supply the case the projector route cannot.

\paragraph{Model expansion and routing.} LLaMA~Pro's block expansion is
identity at initialization and drifts with training
\citep{wu2024llamapro}; Net2Net's function-preserving transforms are
deliberately a warm start \citep{chen2016net2net}; progressive networks
\citep{rusu2016pnn} and side-tuning \citep{zhang2020sidetuning} preserve
frozen components exactly but select paths with task labels at
inference. Multi-adapter serving systems such as S-LoRA
\citep{sheng2024slora} also execute adapters conditionally, hot-swapping
them per request; there the routing signal is deployment metadata
attached to the request, whereas our gate is a function of the input's
own token composition, and none of these systems states or verifies an
exactness property for the unadapted path. We obtain single-model
deployment with a guarantee that is exact, input-conditional, and
permanent. LLaMA~Pro's stated open problem, extending to modalities
while maintaining original ability, is the setting solved here for
embedding models.

\paragraph{Routing and modality experts.} Mixture-of-modality-experts
architectures such as VLMo \citep{bao2022vlmo} route tokens to
modality-specific expert FFNs, and mixture-of-experts routing generally
allocates per-token capacity; in all of these the experts are trained
jointly with the backbone from the start. Our adapters are post-hoc
additions to a frozen, already-deployed base, and the bypass path
carries an exactness guarantee rather than a learned routing decision.

\paragraph{Audio-text embedding models.} Specialist dual encoders set
the pace on audio-text retrieval
\citep{wu2023laionclap,mei2023wavcaps,dinkel2025glap,niizumi2025m2dclap},
and audio-native LLM retrievers push it further with in-backbone LoRA
and re-ranking \citep{omniembedaudio2026,aurola2026}; none serves
modalities beyond audio and text, and none preserves a base model's
behavior. Space-stitching methods \citep{wang2025omnibind,
wang2024freebind} combine pre-trained spaces with trained projectors,
changing every constituent's outputs. ImageBind
\citep{girdhar2023imagebind} and ONE-PEACE \citep{wang2023onepeace}
train unified spaces from scratch rather than extending a deployed one.

\section{Limitations}
\label{sec:limitations}

\paragraph{One base family, two modalities.} Evidence comes from two
modalities, audio and thermal, added to one 2B base through two different
injection routes (connector tokens and the frozen vision path), and the
two packs compose bitwise (\S\ref{sec:thermal}). What remains untested is
generality across base \emph{families}: every result uses the same frozen
Qwen3-VL backbone, so we do not claim the mechanism transfers to a
different embedding LLM without re-verification. Scaling past two
co-loaded packs is likewise unmeasured.

\paragraph{Audio ceiling.} Gains are bounded by what a frozen 2B LM can
do with audio tokens: specialist systems that train both towers remain
well ahead on in-domain audio-text \ratk{1}
(\S\ref{sec:release-results}).

\paragraph{Scope of the bitwise claim.} Bitwise equality is asserted
under matched inference configuration; across kernels or hardware the
model inherits exactly the base's determinism, no more. Functional
identity of the computation graph is unconditional.

\paragraph{Missing ablations.} Layer-subset placement, random-init
adapters, and a trained ungated-adapter arm were not run; the rank
sweep and fine-tune-length checks are the ablations we have
(Appendix~\ref{app:variants}). The controlled gate experiment is
replicated across three seeds with the effect positive in every
direction at every seed (\S\ref{sec:gate-exp}); the full-scale pretrain
pair and the release checkpoints remain single runs, read with the
standard-error convention of \S\ref{sec:exp-setup}.

\paragraph{Claimed-modality inputs are new behavior.} The guarantee
covers inputs no pack claims; a sequence a pack does claim (audio, or a
thermal image), including its text tokens, is adapted and has no base
counterpart. Inputs that would open two gates in one forward are outside
Proposition~\ref{prop:exclusive} and untested.

\section{Conclusion}
\label{sec:conclusion}

Modality-gated deep adapters add a modality to a frozen embedding model
while keeping the model, on every input it previously served, exactly
itself. The guarantee is architectural rather than empirical, survives
arbitrary training, needs no routing metadata, and is verified by
exact-equality tests shipped with the code; the same experiments that
establish it also show the mechanism is where frozen-LLM audio capacity
should go, improving the point estimate in every text-to-audio cell over
a connector-only predecessor after a motivating probe ruled out the
encoder. Adding thermal imagery through a second, dissimilar injection
route, and showing the two packs compose to an exact-zero isolation
matrix, is evidence that the construction is genuinely modality-agnostic
rather than an audio-specific trick. The
released artifacts are the \fetwo{} audio checkpoint (pinned revision
\texttt{v0.1-preview}) at
\url{https://huggingface.co/EximiusLabs/fusion-embedding-2-2b-preview},
the \texttt{ember} thermal pack at
\url{https://huggingface.co/EximiusLabs/fusion-embedding-2-ember},
with Apache-2.0 training and evaluation code, including the invariance
and composability test suites, at
\url{https://github.com/Eximius-Labs/fusion-embedding}.

\appendix

\section{Probe and controlled-experiment protocol}
\label{app:probe}

\paragraph{Tower swap (\S\ref{sec:probe}).} Both arms: 45K
AudioCaps-train-only corpus, 800 steps, identical connector
architecture, loss, optimizer, schedule, and batch size; each arm
ingests and evaluates on frames from its own frozen tower (Qwen2.5-Omni
audio tower vs.\ Dasheng-base). Audio-to-text cells are from the
standard in-run rescore; text-to-audio cells are from the run logs.
Hyperparameters were tuned on the Omni arm and reused unchanged on the
Dasheng arm.
GLAP's encoder numbers quoted for context (AudioCaps text-to-audio
mAP@10: CED-Base 58.6, BEATs 55.1, Dasheng 55.8, Whisper-Base 46.5) are
from their published study \citep{dinkel2025glap}.

\paragraph{Gate experiment (\S\ref{sec:gate-exp}).} Same corpus, steps,
and recipe as the tower-swap Qwen-tower arm, which doubles as the
control row of Table~\ref{tab:gate}. Adapter arms differ from the
control only in attaching gated adapters at the stated rank. The
$\geq{+}3$ \ratk{10} acceptance threshold was fixed in the program's
plan document before the runs. In-run protocol: the automated
AudioCaps rescore at the run's own precision; arms are compared only
within matched configuration. The seed replication of
\S\ref{sec:gate-exp} reran the control and rank-384 arms at two further
seeds each (seeding torch, numpy, and the loader shuffle; the seed is
part of the run's resume key); the full per-arm grids are in the
released result record \texttt{adapter\_probe\_multiseed.json}.

\paragraph{Full-scale reproduction.} 518{,}183-pair corpus (592K minus
73{,}716 junk-metadata clips), soft-label and false-negative-masking
loss terms, 3{,}900 steps, effective batch 1{,}024, full-corpus frozen
text bank; the baseline is the closest available one, the 592K
raw-corpus run without adapters or the loss terms (in-run \ratk{10}
0.674, matched evaluation protocol), so the pair differs on the three
axes noted in \S\ref{sec:gate-exp}. The release
lineage is bf16 end to end; a 4-bit lineage of the same recipe measures
about 2 points lower when scored at bf16, which is why lineages are
never mixed.

\section{Verification artifacts}
\label{app:verification}

The gate is \texttt{AdapterGate} in
\texttt{fusion\_embedding/adapters.py}: a depth-counted context manager
whose hooks return the frozen layer's output before any adapter
arithmetic when closed (Figure~\ref{fig:code}). Adapters are owned by
the model wrapper, so the base's parameter snapshot and
\texttt{state\_dict} never contain them. \texttt{tests/test\_adapters.py}
asserts: (i) bitwise text invariance through base and adapter-attached
model with the gate closed, and exact identity of a zero-initialized
stack under an open gate; (ii) audio-gated forwards change while a text
forward in the same process stays bit-identical; (iii) gradient
isolation (only adapter and connector parameters receive gradients);
(iv) checkpoint resume round-trips adapters and refuses
adapter-presence mismatches; (v) warm-start semantics from a
connector-only checkpoint; (vi) a gate that fails to span the re-run
forwards of gradient checkpointing raises \texttt{CheckpointError}
loudly. The training loop asserts \texttt{base\_drift} $= 0$ (maximum
absolute change over every base parameter) after every training run.
The release smoke test loads the packaged checkpoint from the public
repository and embeds identical text and image inputs twice, once with
the adapter hooks attached (gate closed) and once with every hook
removed, requiring \texttt{torch.equal} on both modalities; it also
exercises the audio path end to end through the public loading code.
The gate contract is enforced at the entry points on both sides of the
release: \texttt{encode\_text} in the training code and
\texttt{embed\_image} in the released inference code (the vision path
shares the hooked decoder layers) each raise if called under an open
gate, so neither cached text targets nor image embeddings can be
silently adapted.

\section{Variants and checks not run}
\label{app:variants}

In the spirit of reporting what was and was not tested: the rank sweep
(128, 384) and the fine-tune-length check (a 200-step arm scored
0.273/0.695 against 0.302/0.743 at 400 steps under the release-protocol
rescore of its lineage; a 600-step arm on the predecessor showed the
other side of the optimum) are the ablations we have. Not run: adapters on layer
subsets, randomly initialized (non-zero) adapters, a trained
ungated-adapter arm (its inference-time effect on non-audio inputs is
what Table~\ref{tab:verify}'s open-gate row makes visible), and
multi-seed replications. None of these is blocked by the design; they
were out of budget, and we state them rather than imply them.

\section{Hyperparameters}
\label{app:hyper}

Connector: perceiver-resampler, internal width 384, 64 latent queries,
6 blocks, 16.4M parameters. Adapters (Eq.~\ref{eq:adapter}): rank 384,
SiLU, LayerNorm-first
bottleneck, zero-init up-projection, fp32 compute cast back to stream
dtype, 44.2M parameters over 28 layers. Loss: symmetric InfoNCE at
every Matryoshka rung with a light covariance penalty; audio-to-text
negatives augmented with the full-corpus frozen-text bank; soft labels
($\beta{=}0.3$) and false-negative masking ($\tau{=}0.98$) at 500K+
scale. Optimizer AdamW, cosine decay, 5\% warmup. Training input
format is the base's native chat template; the format study and full
recipe details are in the companion technical report~\citep{tonmoy2026fusion}.


\begin{thebibliography}{39}
\providecommand{\natexlab}[1]{#1}

\bibitem[{Alayrac et~al.(2022)Alayrac, Donahue, Luc, Miech, Barr, Hasson, Lenc, Mensch, Millican, Reynolds et~al.}]{alayrac2022flamingo}
Jean-Baptiste Alayrac, Jeff Donahue, Pauline Luc, Antoine Miech, Iain Barr, Yana Hasson, Karel Lenc, Arthur Mensch, Katherine Millican, Malcolm Reynolds, et~al. 2022.
\newblock Flamingo: a visual language model for few-shot learning.
\newblock \emph{Advances in Neural Information Processing Systems (NeurIPS)}.
\newblock ArXiv:2204.14198.

\bibitem[{Bachlechner et~al.(2021)Bachlechner, Majumder, Mao, Cottrell, and McAuley}]{bachlechner2021rezero}
Thomas Bachlechner, Bodhisattwa~Prasad Majumder, Huanru~Henry Mao, Garrison~W. Cottrell, and Julian McAuley. 2021.
\newblock {ReZero} is all you need: Fast convergence at large depth.
\newblock In \emph{Uncertainty in Artificial Intelligence (UAI)}.
\newblock ArXiv:2003.04887.

\bibitem[{Bao et~al.(2022)Bao, Wang, Dong, Liu, Mohammed, Aggarwal, Som, and Wei}]{bao2022vlmo}
Hangbo Bao, Wenhui Wang, Li~Dong, Qiang Liu, Owais~Khan Mohammed, Kriti Aggarwal, Subhojit Som, and Furu Wei. 2022.
\newblock {VLMo}: Unified vision-language pre-training with mixture-of-modality-experts.
\newblock In \emph{Advances in Neural Information Processing Systems (NeurIPS)}.
\newblock ArXiv:2111.02358.

\bibitem[{Chen et~al.(2020)Chen, Xie, Vedaldi, and Zisserman}]{chen2020vggsound}
Honglie Chen, Weidi Xie, Andrea Vedaldi, and Andrew Zisserman. 2020.
\newblock Vggsound: A large-scale audio-visual dataset.
\newblock In \emph{IEEE International Conference on Acoustics, Speech and Signal Processing (ICASSP)}.
\newblock ArXiv:2004.14368.

\bibitem[{Chen et~al.(2016)Chen, Goodfellow, and Shlens}]{chen2016net2net}
Tianqi Chen, Ian Goodfellow, and Jonathon Shlens. 2016.
\newblock {Net2Net}: Accelerating learning via knowledge transfer.
\newblock In \emph{International Conference on Learning Representations (ICLR)}.
\newblock ArXiv:1511.05641.

\bibitem[{Dinkel et~al.(2026)Dinkel, Yan, Wang, Wang, Sun, Niu, Liu, Li, Zhang, and Luan}]{dinkel2025glap}
Heinrich Dinkel, Zhiyong Yan, Tianzi Wang, Yongqing Wang, Xingwei Sun, Yadong Niu, Jizhong Liu, Gang Li, Junbo Zhang, and Jian Luan. 2026.
\newblock Glap: General contrastive audio-text pretraining across domains and languages.
\newblock In \emph{IEEE International Conference on Acoustics, Speech and Signal Processing (ICASSP)}.
\newblock ArXiv:2506.11350.

\bibitem[{Dinkel et~al.(2024)Dinkel, Yan, Wang, Zhang, Wang, and Wang}]{dinkel2024dasheng}
Heinrich Dinkel, Zhiyong Yan, Yongqing Wang, Junbo Zhang, Yujun Wang, and Bin Wang. 2024.
\newblock Scaling up masked audio encoder learning for general audio classification.
\newblock In \emph{Proceedings of Interspeech}.
\newblock ArXiv:2406.06992.

\bibitem[{Drossos et~al.(2020)Drossos, Lipping, and Virtanen}]{drossos2020clotho}
Konstantinos Drossos, Samuel Lipping, and Tuomas Virtanen. 2020.
\newblock Clotho: an audio captioning dataset.
\newblock In \emph{IEEE International Conference on Acoustics, Speech and Signal Processing (ICASSP)}.
\newblock ArXiv:1910.09387. Evaluation split v2.1 from Zenodo record 4783391.

\bibitem[{Girdhar et~al.(2023)Girdhar, El-Nouby, Liu, Singh, Alwala, Joulin, and Misra}]{girdhar2023imagebind}
Rohit Girdhar, Alaaeldin El-Nouby, Zhuang Liu, Mannat Singh, Kalyan~Vasudev Alwala, Armand Joulin, and Ishan Misra. 2023.
\newblock Imagebind: One embedding space to bind them all.
\newblock In \emph{IEEE/CVF Conference on Computer Vision and Pattern Recognition (CVPR)}.
\newblock ArXiv:2305.05665.

\bibitem[{Gong et~al.(2023)Gong, Khurana, Karlinsky, and Glass}]{gong2023whisperat}
Yuan Gong, Sameer Khurana, Leonid Karlinsky, and James Glass. 2023.
\newblock Whisper-at: Noise-robust automatic speech recognizers are also strong general audio event taggers.
\newblock In \emph{Proceedings of Interspeech}.
\newblock ArXiv:2307.03183.

\bibitem[{H{\"o}nicke et~al.(2026)H{\"o}nicke, G{\"u}nther, Koukounas, Akram, Martens, Sturua, and Xiao}]{jina2026omni}
Florian H{\"o}nicke, Michael G{\"u}nther, Andreas Koukounas, Mohammad~Kalim Akram, Scott Martens, Saba Sturua, and Han Xiao. 2026.
\newblock jina-embeddings-v5-omni: Geometry-preserving embeddings via locked aligned towers.
\newblock \emph{arXiv preprint arXiv:2605.08384}.

\bibitem[{Houlsby et~al.(2019)Houlsby, Giurgiu, Jastrzebski, Morrone, de~Laroussilhe, Gesmundo, Attariyan, and Gelly}]{houlsby2019adapters}
Neil Houlsby, Andrei Giurgiu, Stanislaw Jastrzebski, Bruna Morrone, Quentin de~Laroussilhe, Andrea Gesmundo, Mona Attariyan, and Sylvain Gelly. 2019.
\newblock Parameter-efficient transfer learning for {NLP}.
\newblock In \emph{Proceedings of the 36th International Conference on Machine Learning (ICML)}.
\newblock ArXiv:1902.00751.

\bibitem[{Hu et~al.(2022)Hu, Shen, Wallis, Allen-Zhu, Li, Wang, and Chen}]{hu2022lora}
Edward~J. Hu, Yelong Shen, Phillip Wallis, Zeyuan Allen-Zhu, Yuanzhi Li, Shean Wang, and Weizhu Chen. 2022.
\newblock {LoRA}: Low-rank adaptation of large language models.
\newblock In \emph{International Conference on Learning Representations (ICLR)}.
\newblock ArXiv:2106.09685.

\bibitem[{Jia et~al.(2021)Jia, Zhu, Li, Tang, Liu, and Zhou}]{jia2021llvip}
Xinyu Jia, Chuang Zhu, Minzhen Li, Wenqi Tang, Shengjie Liu, and Wenli Zhou. 2021.
\newblock {LLVIP}: A visible-infrared paired dataset for low-light vision.
\newblock In \emph{IEEE/CVF International Conference on Computer Vision (ICCV) Workshops}.
\newblock ArXiv:2108.10831.

\bibitem[{Kim et~al.(2019)Kim, Kim, Lee, and Kim}]{kim2019audiocaps}
Chris~Dongjoo Kim, Byeongchang Kim, Hyunmin Lee, and Gunhee Kim. 2019.
\newblock Audiocaps: Generating captions for audios in the wild.
\newblock In \emph{Proceedings of NAACL-HLT}.

\bibitem[{Kusupati et~al.(2022)Kusupati, Bhatt, Rege, Wallingford, Sinha, Ramanujan, Howard-Snyder, Chen, Kakade, Jain, and Farhadi}]{kusupati2022mrl}
Aditya Kusupati, Gantavya Bhatt, Aniket Rege, Matthew Wallingford, Aditya Sinha, Vivek Ramanujan, William Howard-Snyder, Kaifeng Chen, Sham Kakade, Prateek Jain, and Ali Farhadi. 2022.
\newblock Matryoshka representation learning.
\newblock In \emph{Advances in Neural Information Processing Systems (NeurIPS)}.
\newblock ArXiv:2205.13147.

\bibitem[{Li et~al.(2023)Li, Li, Savarese, and Hoi}]{li2023blip2}
Junnan Li, Dongxu Li, Silvio Savarese, and Steven Hoi. 2023.
\newblock {BLIP-2}: Bootstrapping language-image pre-training with frozen image encoders and large language models.
\newblock In \emph{Proceedings of the 40th International Conference on Machine Learning (ICML)}.
\newblock ArXiv:2301.12597.

\bibitem[{Li et~al.(2026)Li, Zhang, Long, Chen, Song, Bai, Yang, Xie, Yang, Liu, Zhou, and Lin}]{qwen3vlembedding}
Mingxin Li, Yanzhao Zhang, Dingkun Long, Keqin Chen, Sibo Song, Shuai Bai, Zhibo Yang, Pengjun Xie, An~Yang, Dayiheng Liu, Jingren Zhou, and Junyang Lin. 2026.
\newblock Qwen3-vl-embedding and qwen3-vl-reranker: A unified framework for state-of-the-art multimodal retrieval and ranking.
\newblock \emph{arXiv preprint arXiv:2601.04720}.

\bibitem[{Liu et~al.(2024)Liu, Li, Zhang, Dinkel, Wang, Yan, Wang, and Wang}]{interspeech2024cedwhisper}
Jizhong Liu, Gang Li, Junbo Zhang, Heinrich Dinkel, Yongqing Wang, Zhiyong Yan, Yujun Wang, and Bin Wang. 2024.
\newblock Enhancing automated audio captioning via large language models with optimized audio encoding.
\newblock In \emph{Proceedings of Interspeech}.
\newblock ArXiv:2406.13275.

\bibitem[{Mei et~al.(2024)Mei, Meng, Liu, Kong, Ko, Zhao, Plumbley, Zou, and Wang}]{mei2023wavcaps}
Xinhao Mei, Chutong Meng, Haohe Liu, Qiuqiang Kong, Tom Ko, Chengqi Zhao, Mark~D. Plumbley, Yuexian Zou, and Wenwu Wang. 2024.
\newblock Wavcaps: A chatgpt-assisted weakly-labelled audio captioning dataset for audio-language multimodal research.
\newblock \emph{IEEE/ACM Transactions on Audio, Speech, and Language Processing}.
\newblock ArXiv:2303.17395.

\bibitem[{Meng et~al.(2025)}]{meng2025vlm2vecv2}
Rui Meng et~al. 2025.
\newblock {VLM2Vec-V2}: Advancing multimodal embedding for videos, images, and visual documents.
\newblock \emph{arXiv preprint arXiv:2507.04590}.
\newblock The MMEB-V2 benchmark.

\bibitem[{Niizumi et~al.(2025)Niizumi, Takeuchi, Ohishi, Harada, and Kashino}]{niizumi2025m2dclap}
Daisuke Niizumi, Daiki Takeuchi, Yasunori Ohishi, Noboru Harada, and Kunio Kashino. 2025.
\newblock M2d-clap: Exploring general-purpose audio-language representations beyond clap.
\newblock \emph{arXiv preprint arXiv:2503.22104}.

\bibitem[{Panagopoulou et~al.(2024)Panagopoulou, Xue, Yu, Li, Li, Joty, Xu, Savarese, Xiong, and Niebles}]{panagopoulou2024xinstructblip}
Artemis Panagopoulou, Le~Xue, Ning Yu, Junnan Li, Dongxu Li, Shafiq Joty, Ran Xu, Silvio Savarese, Caiming Xiong, and Juan~Carlos Niebles. 2024.
\newblock {X-InstructBLIP}: A framework for aligning x-modal instruction-aware representations to {LLMs} and emergent cross-modal reasoning.
\newblock In \emph{European Conference on Computer Vision (ECCV)}.
\newblock ArXiv:2311.18799.

\bibitem[{Piczak(2015)}]{piczak2015esc50}
Karol~J. Piczak. 2015.
\newblock Esc: Dataset for environmental sound classification.
\newblock In \emph{Proceedings of the 23rd ACM International Conference on Multimedia}.

\bibitem[{Radford et~al.(2023)Radford, Kim, Xu, Brockman, McLeavey, and Sutskever}]{radford2022whisper}
Alec Radford, Jong~Wook Kim, Tao Xu, Greg Brockman, Christine McLeavey, and Ilya Sutskever. 2023.
\newblock Robust speech recognition via large-scale weak supervision.
\newblock In \emph{Proceedings of the 40th International Conference on Machine Learning (ICML)}.
\newblock ArXiv:2212.04356.

\bibitem[{Rusu et~al.(2016)Rusu, Rabinowitz, Desjardins, Soyer, Kirkpatrick, Kavukcuoglu, Pascanu, and Hadsell}]{rusu2016pnn}
Andrei~A. Rusu, Neil~C. Rabinowitz, Guillaume Desjardins, Hubert Soyer, James Kirkpatrick, Koray Kavukcuoglu, Razvan Pascanu, and Raia Hadsell. 2016.
\newblock Progressive neural networks.
\newblock \emph{arXiv preprint arXiv:1606.04671}.

\bibitem[{Sheng et~al.(2024)Sheng, Cao, Li, Hooper, Lee, Yang, Chou, Zhu, Zheng, Keutzer, Gonzalez, and Stoica}]{sheng2024slora}
Ying Sheng, Shiyi Cao, Dacheng Li, Coleman Hooper, Nicholas Lee, Shuo Yang, Christopher Chou, Banghua Zhu, Lianmin Zheng, Kurt Keutzer, Joseph~E. Gonzalez, and Ion Stoica. 2024.
\newblock {S-LoRA}: Serving thousands of concurrent {LoRA} adapters.
\newblock In \emph{Proceedings of Machine Learning and Systems (MLSys)}.
\newblock ArXiv:2311.03285.

\bibitem[{Tonmoy(2026)}]{tonmoy2026axis}
Abdul~Basit Tonmoy. 2026.
\newblock Discriminative axis, not data volume: What a contrastive corpus teaches an audio embedding.
\newblock \emph{arXiv preprint arXiv:2608.01560}.

\bibitem[{Tonmoy et~al.(2026)Tonmoy, Hoque, Arham, and Luthra}]{tonmoy2026fusion}
Abdul~Basit Tonmoy, Kazi~Fardinul Hoque, Md. Shahrier~Islam Arham, and Arman Luthra. 2026.
\newblock Fusion embedding: A unified embedding space for text, image, video, and audio.
\newblock \emph{arXiv preprint arXiv:2607.18666}.

\bibitem[{van~den Oord et~al.(2018)van~den Oord, Li, and Vinyals}]{oord2018infonce}
A{\"a}ron van~den Oord, Yazhe Li, and Oriol Vinyals. 2018.
\newblock Representation learning with contrastive predictive coding.
\newblock \emph{arXiv preprint arXiv:1807.03748}.

\bibitem[{Wang et~al.(2023)Wang, Wang, Lin, Bai, Zhou, Zhou, Wang, and Zhou}]{wang2023onepeace}
Peng Wang, Shijie Wang, Junyang Lin, Shuai Bai, Xiaohuan Zhou, Jingren Zhou, Xinggang Wang, and Chang Zhou. 2023.
\newblock {ONE-PEACE}: Exploring one general representation model toward unlimited modalities.
\newblock \emph{arXiv preprint arXiv:2305.11172}.

\bibitem[{Wang et~al.(2024)Wang, Zhang, Cheng, Huang, Liu, Ye, Huang, Zhao, Jin, Gao, and Zhao}]{wang2024freebind}
Zehan Wang, Ziang Zhang, Xize Cheng, Rongjie Huang, Luping Liu, Zhenhui Ye, Haifeng Huang, Yang Zhao, Tao Jin, Peng Gao, and Zhou Zhao. 2024.
\newblock Freebind: Free lunch in unified multimodal space via knowledge fusion.
\newblock In \emph{International Conference on Machine Learning (ICML)}.
\newblock ArXiv:2405.04883.

\bibitem[{Wang et~al.(2025)Wang, Zhang, Hong, Zhang, Liu, Huang, Cheng, Ji, Jin, Zhao, and Zhao}]{wang2025omnibind}
Zehan Wang, Ziang Zhang, Minjie Hong, Hang Zhang, Luping Liu, Rongjie Huang, Xize Cheng, Shengpeng Ji, Tao Jin, Hengshuang Zhao, and Zhou Zhao. 2025.
\newblock Omnibind: Large-scale omni multimodal representation via binding spaces.
\newblock In \emph{International Conference on Learning Representations (ICLR)}.
\newblock ArXiv:2407.11895.

\bibitem[{Wu et~al.(2024)}]{wu2024llamapro}
Chengyue Wu et~al. 2024.
\newblock {LLaMA} pro: Progressive {LLaMA} with block expansion.
\newblock In \emph{Proceedings of the Annual Meeting of the Association for Computational Linguistics (ACL)}.
\newblock ArXiv:2401.02415.

\bibitem[{Wu et~al.(2023)Wu, Chen, Zhang, Hui, Berg-Kirkpatrick, and Dubnov}]{wu2023laionclap}
Yusong Wu, Ke~Chen, Tianyu Zhang, Yuchen Hui, Taylor Berg-Kirkpatrick, and Shlomo Dubnov. 2023.
\newblock Large-scale contrastive language-audio pretraining with feature fusion and keyword-to-caption augmentation.
\newblock In \emph{IEEE International Conference on Acoustics, Speech and Signal Processing (ICASSP)}.
\newblock ArXiv:2211.06687.

\bibitem[{Xu et~al.(2026)Xu, Thom{\'e}, Horak, Xie, and Zisserman}]{aurola2026}
Jilan Xu, Carl Thom{\'e}, Danijela Horak, Weidi Xie, and Andrew Zisserman. 2026.
\newblock Scaling audio-text retrieval with multimodal large language models.
\newblock \emph{arXiv preprint arXiv:2602.18010}.

\bibitem[{Xu et~al.(2025)}]{xu2025qwen25omni}
Jin Xu et~al. 2025.
\newblock Qwen2.5-omni technical report.
\newblock \emph{arXiv preprint arXiv:2503.20215}.

\bibitem[{Yoo et~al.(2026)Yoo, Shin, Lee, Koo, and Chang}]{omniembedaudio2026}
HaeJun Yoo, Yongseop Shin, Insung Lee, Myoung-Wan Koo, and Du-Seong Chang. 2026.
\newblock Omni-embed-audio: Leveraging multimodal llms for robust audio-text retrieval.
\newblock In \emph{Proceedings of the Annual Meeting of the Association for Computational Linguistics (ACL)}.
\newblock ArXiv:2604.18360.

\bibitem[{Zhang et~al.(2020)Zhang, Sax, Zamir, Guibas, and Malik}]{zhang2020sidetuning}
Jeffrey~O. Zhang, Alexander Sax, Amir Zamir, Leonidas Guibas, and Jitendra Malik. 2020.
\newblock Side-tuning: A baseline for network adaptation via additive side networks.
\newblock In \emph{European Conference on Computer Vision (ECCV)}.
\newblock ArXiv:1912.13503.

\end{thebibliography}
\end{document}